\documentclass{article}
\usepackage[utf8]{inputenc}
\pdfoutput=1
\usepackage{tikz}
\usetikzlibrary{matrix,arrows,decorations.pathmorphing}
\usepackage{tikz-cd}
\usepackage{amsmath,amsfonts,amsthm,amssymb}
\usepackage{dsfont}
\usepackage{graphicx}

\usepackage{algorithm}
\usepackage[noend]{algpseudocode}

\usepackage{comment}

\title{Geometric Decomposition of Feed Forward Neural Networks}
\author{Sven Cattell}

\newcommand{\R}{\mathbb{R}}

\newcommand{\aA}{\mathcal{A}}
\newcommand{\aR}{\mathcal{R}}

\newcommand{\aS}{\mathcal{S}}

\newcommand{\ah}{\mathfrak{h}}

\newcommand{\U}{\mathcal{U}}

\newcommand{\Chi}{\mathds{1}}

\makeatletter
\def\BState{\State\hskip-\ALG@thistlm}
\makeatother

\floatname{algorithm}{Procedure}

\newtheorem{defi}{Definition}

\newtheorem{theorem}{Theorem}[section]

\newtheorem{lemma}[theorem]{Lemma}

\begin{document}

\maketitle
\begin{abstract}
There have been several attempts to mathematically understand neural networks and many more from biological and computational perspectives. The field has exploded in the last decade, yet neural networks are still treated much like a black box. In this work we describe a structure that is inherent to a feed forward neural network. This will provide a framework for future work on neural networks to improve training algorithms, compute the homology of the network, and other applications. Our approach takes a more geometric point of view and is unlike other attempts to mathematically understand neural networks that rely on a functional perspective. 
\end{abstract}

\section{Introduction}

In the last decade deep learning has exploded. With many applications to image recognition, natural language processing, and many other pattern recognition tasks that have traditionally had poor performance with other machine learning techniques seem to be finally cracked. However, deep learning is rather poorly understood mathematically. In the original papers the intuition provided is compelling and has proved to be useful in developing the technique \cite{hinton2006reducing}\cite{bengio2005convex}\cite{lecun2015deep}. However, this needs some mathematical backing in order to properly use and improve the techniques. 

This paper aims to give a description of the parts of a neural network and their functions. In order to simplify the structure of a fully connected feed forward neural network to one that is easily understood we will mainly work with the step edge activation function. This may seem like a step back from the modern use of smooth or piecewise linear activation functions, but if the activation function we want to train against is a approximation of the step edge function (like the sigmoidal function), then we can still make strong inferences of the general structure. This paper focuses on the geometry of the network so, unless otherwise noted, we are using the step-edge activation function everywhere. The main theorem of this paper is:

\begin{theorem}
A binary classification neural network is an indicator function on a union of some regions in $\R^n - \bigcup_{P \in \mathcal{A}} P$, where $\mathcal{A}$ is defined by the first layer.
\end{theorem}

This is not similar to previous work in mathematically understanding shallow neural networks of Kolmogorov \cite{Kolmogorov}, Funahashi \cite{Funahashi:1989:ARC:71287.71290}or Sprecher \cite{sprecher1965structure}. They prove that there is some network that can approximate a given function. However this is non-constructive and does not apply to deep networks. The more recent results of \cite{2015arXiv150907385S} and \cite{2016arXiv160707110C} construct well understood networks of arbitrary precision. This paper is intended to give a more general structure theory that can be applied to build networks or manipulate them.

\section{Background}

\subsection{Neural Networks}
This part of the background is intended for mathematicians unfamiliar with neural networks. Feed forward neural networks are easily expressed as a composition of linear functions, with a non-linear \emph{activation function}. We will denote the activation function by $\ah$. It is always defined as a single variable function. However, it is regularly referred to as being defined on $\R^n$, not just $\R$, this is just $\ah^{\oplus n}$, or $\ah$ on each coordinate of $\R^n$. 
\begin{defi}\label{nndefi}
A \emph{classification neural network} of height $h$ with $n$ inputs and $k$ outputs is a function $N:\R^n \to \R^k$ such that it can be written as a composition of $h+1$ functions of the form $\ah(A(x)+b)$ where $A$ is some matrix $b$ is an offset vector called \emph{layers}. In other words there exists $h+1$ matrices $A_i$, with offset vectors $b_i$, and:
$$ N(x) = \ah(A_{h+1}\ah(A_h(\dots \ah(A_2\ah(A_1(x)+b_1)+b_2)\dots)+b_h)+b_{h+1})$$
Each $A_i$ is a matrix of size $k_{i} \times k_{i-1}$ and $b_i$ is a $k_{i}$-tuple, we say the layer $i$ has $k_i$ nodes. 
\end{defi}
Of course, this is a rather obtuse definition, it is usually expressed as a directed graph. The nodes are organized into layers with directed edges pointing up the layers. These edges are given weights and each node, other than the input nodes, are a linear sum of all nodes that have an edge pointing towards the aforementioned node. Sometimes the offset is represented by an additional input node who is connected to all the nodes in higher layers and has a fixed value of $1$, but here we shall think of the offset as a part of each node that isn't a input node.

For simplicity, we will mainly discuss a neural network with one output. All the results generalize quite easily to multiple outputs, so they are mostly omitted.

\subsection{Hyperplane Arrangements}
This section of the background is intended for those familiar with neural networks, but unfamiliar with the language of hyperplane arrangements. It is intended to be a short summary of the necessary language required for the main result. The following is mostly lifted from \cite{hyperplanearranegement}, which has much more detail that we may provide here. 

An \emph{arrangement of hyperplanes}, $\aA$ in $\R^n$ is a finite set of codimension 1 affine subspaces. The \emph{regions} between the hyperplanes refer to the components of $\R^n - \bigcup_{P \in \aA} P$. An arrangement is in \emph{general position} if a small perturbation of the hyperplanes in the arrangement does not change the number of hyperplanes. For example two parallel lines in $\R^2$ are not in general position. An arbitrarily small rotation of one of the lines will result in the number of regions increasing from $3$ to $4$. The other essential structure inherent in hyperplane arrangements 

\begin{defi}
The intersection poset of an arrangement of $k$ hyperplanes, $\aA$ is the set 
$$ L(\aA)=\{ x \subset \{ 1,\dots ,k\} | \R^n \cap \bigcap_{i \in x} P_i \neq \emptyset  \} $$
where $P_i \in \aA $. This is equipped with the partial order $x \leq y$ if $x \subset y$.
\end{defi}

A poset is a set equipped with a partial relation, i.e. if $x,y \in X$ then we may have $x \neq y$, $x \not< y$ and $y \not< x$. In a poset, an element $y$ \emph{covers} $x$, or $x \lessdot y$, if there is no $z$ different from $x$ and $y$ such that $x<z<y$. We can draw the Hasse diagram of a poset by drawing a node for each element of the poset and a directed edge from the node for $x$ to the node for $y$ if $x \lessdot y$.

\section{Characterization of a Perceptron}

If we examine the first layer of a neural network each node corresponds to a weighted sum of the input nodes with an offset. We can represent the weighted sum by a vector $v$, and the offset by $b$. Then the node is ``active'' on the point $x$, and outputting a $1$ if $v\cdot x > b$, and $0$ otherwise. This is the same as the indicator function on the positive component, $U$, of $\R^n -P$ where $P$ is the associated hyperplane the following set:
$$ U=\{x \in \R^n | v \cdot x > b  \}. $$
So, another way of expressing a node on the first layer is by the indicator function on the set $U$, $\Chi_U$. 

Therefore the input, under the step edge activation function, to the second layer of the network is the output of a collection of $n$ indicator functions. The second layer has a binary input on $n$ variables for each point in the underlying space. Therefore all it can do is make decisions based on which side of each hyperplane the given input point is. We will use $\mathcal{A}$ to denote a collection of hyperplanes in $R^n$. This divides $\R^n$ into regions, $\aR$, the connected components of $\R^n - \bigcup_{P \in \mathcal{A}} P$. The second layer is only aware of which region you are in as each region produces a unique signature output of the first layer. The subsequent layers are a means of making a choice of which regions to include, each layer past the first will amount a process we call a weighted union of the sets associated to the nodes of the previous layer. 

\subsection{Regions in a Polarized Arrangement}
A plane in $\R^n$ is defined by a normal vector $v$ and a offset value $b$. If we require that the normal vector be of length 1, then there are two normal vectors to choose from. Both work, but they define a different orientation of the hyperplane. Therefore, rather than just having two components of $\R^n$ we can now distinguish between them.

\begin{defi}
A polarization of a hyperplane in $\R^n$ for the plane $P$ with normal vector $v$ and offset $b$ are the two sets:
$$ \aR_P^{+} = \{x | x \cdot v + b> 0 \} $$
$$ \aR_P^{-} = \{x | x \cdot v + b< 0 \} $$
We call $\aR_P^+$ the positive side and $U_P^-$ the negative side of the plane.
\end{defi}

We usually index our hyperplanes by the numbers $\{ 1,2, \dots k\}$ so we write $\aR_{i}^+$ for the positive side of the $i$th hyperplane in this case. Nodes in the first layer of the neural network are the \emph{hyperplane layer} as they determine a polarized arrangement of hyperplanes. 

For a polarizes arrangement of hyperplanes $\aA$ indexed by a set $I=\{ 1,2, \dots k\}$ we can define and label the regions. 

\begin{defi}
The regions, $\aR$ of a polarization arrangement of hyperplanes $\aA$ is the set of convex polytopes formed by taking all possible intersections of the positive and negatives sides of an ordered set of $k$ plane partitions. Each region is labeled by a $J \subset I$ where 
$$\aR_J = \bigcap_{j\in J} \aR_j^+ \cap \bigcap_{j \in I-J} \aR_j^- $$
\end{defi} 

We are using a similar definition to \cite{MEISER1993286}, but with a subset of $\{1, \dots, n\}$ instead of a ordered $n$-tuple. The two methods are isomorphic, but we choose to use sets for now. There are $2^k$ different labelings one could have for a hyperplane arrangement, depending on the polarization. See figure \ref{repolarization} for an example. 

\begin{figure}\label{repolarization}
\centering
\includegraphics[width=4in]{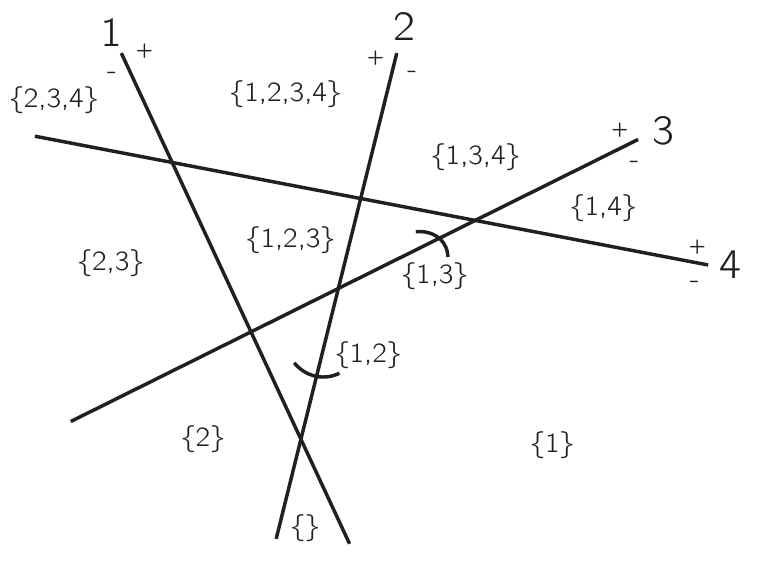}
\includegraphics[width=4in]{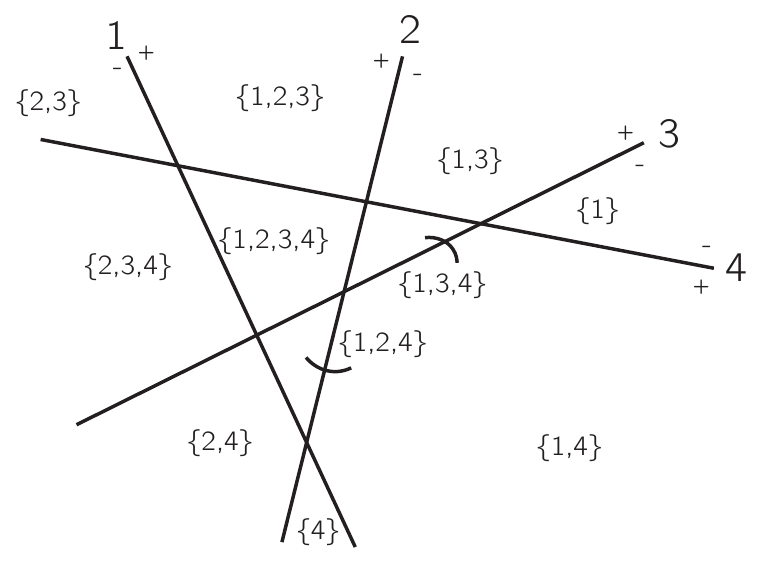}
\caption{Two different labellings of the same regions of an arrangement of hyperplanes. The difference is the plane labeled by $4$ changes polarity. Note, no one region needs to be labeled with $\emptyset$ unlike the labeling from \cite{regionlabeling}}
\end{figure}

Note, not all possible labellings are non-empty regions as the largest possible number of regions in for an arrangement of $k$ hyperplanes $\R^n$ is:
$$ r(\aA) = 1 + k + {k \choose 2} + \dots {k \choose n}. $$
Which is $2^n$ if $k=n$, but when $k>n$ is strictly less than $2^k$. This is when the regions are in general position, that is we can change the values of all the normal vectors and offsets by a small amount and the number of regions will not change. The initial layer of a perceptron will most likely be in general position. We call an index $J \subset I$ trivial if $\aR_J = \emptyset$.

\subsection{Weighted Unions and Selection Layers}

Layers of the neural network after the first layer will amount to a weighted union of the sets associated to the previous layer. The second layer is a weighted union of the positive sides of the polarized hyperplanes generated by the first layer. This results in each node being equivalent to a union of the regions of the arrangement. Thus we call a layer a \emph{selection layer} if it is not the first layer. We start with a set level definition of what each node in a selection layer:

\begin{defi}
A weighted union of subsets of $X$, $\{U_i\}_{i \in I}$, with weights $a_i$ and offset $b$ is
$$ \bigcup_{i\in I}(a_i U_i)-b = \left \{ x | \ah ( \sum_{i \in I} a_i \Chi_{U_i}(x) -b ) > 0 \right \}. $$
The characteristic function is defined on $X$.
\end{defi}

It is clear that this is just taking the output of a node and converting it back to the associated set. In order to manipulate a selection node we need a clear understanding of the weighted union. For a single set $U\subset X$, we can take the compliment of $U$ by the weighted union: $-1 U -\frac{1}{2}$. We can also take the union and intersection of two or more sets:
$$ \bigcap_{i=1}^n U_i = \bigcup_{i=1}^n 1 U_i - (n-\frac{1}{2}),\quad \quad \bigcup_{i=1}^n U_i = \bigcup_{i=1}^n 1 U_i - (\frac{1}{2}).$$
The following lemma demonstrates some of the limitations of the weighted union:

\begin{lemma}\label{1weightedunion}
A weighted union of sets $U_i$ can be written as a union a subset of all intersections of those sets and their complements.
\end{lemma}
\begin{proof}
Let $W=\bigcup_{i\in I}(a_i U_i)-b$ be a weighted union. For each $a_i$,
$$ a_i \Chi_{U_i}(x) = -a_i \Chi_{X-U_i}(x) + 2 $$
therefore without loss of generality we may assume for all $i$, $a_i>0$, else replace $a_i$ with $-a_i$, add $a_i$ to $b$ and replace $U_i$ with $X-U_i$. Let $\tilde{b}$ be the adjusted offset
For $J \subset I$ Let 
$$ U_J =\bigcap_{j\in J} U_j \cap \bigcap_{j \in I-J} (X-U_j), $$
The map 
$$\sigma: P(I) \to \R, \quad \sigma(J)\mapsto \sum_{j\in J} a_j,$$
is order preserving as all $a_i$ are non-negative. Let $\{x>\tilde{b}\}$ be the obvious shorthand for $\{x\in \R|x>\tilde{b}\}$, it is easy to show: 
$$ W = \bigcup_{J \in \sigma^{-1}(\{ x>\tilde{b}\})} U_J. $$
\end{proof}

\begin{defi} For a collection of sets $\{U_i\}_{i\in I}$, and for each $J \in P(I)$ the we get a \emph{region} of the 
$$ U_J =\bigcap_{j\in J} U_j \cap \bigcap_{j \in I-J} (X-U_j) $$
The indexing of the region $U_J$ by $J$ is the standard indexing
\end{defi}

We can see from the proof that if we want to find out what set operations are possible with the weighted union we can restrict ourselves to a union of some intersections of our sets. This can be seen through the light of a $n$-ary logical operation that is strictly composed out of a series of `or's on a collection of `and's. This cannot produce all $n$-ary boolean statements \cite{enderton2001mathematical}. 

The proof provides a way to translate the weights and offset to a map on the poset $P(I)$. However the mapping is very dependent on the signs of the original weights. Given two different weighted unions onto the same space we not only have a $\sigma_1$ and a $\sigma_2$, but we also have two different sign corrections. This means that it's more difficult to compare two selection nodes using this technique. We can however remove the sign correction for the polarization after we have determined the which intersections were selected.

To standardize the polarization we define the following. For any map $p : I \to \{+,-\}$ we may define a self map $r_p:P(I)\to P(I)$ where $j \in r_p(J)$ if $j \in J$ and $p(j)=+$ or if $j \not\in J$ and $p(j) = -$. $r_p$ is clearly a bijection for all $p$. For a collection of sets $\{U_i\}_{i\in I}$ and each $p$ we get a different polarization of all possible intersections of the sets and their complements. The $p$ polarization of the regions of $\{U_i \}_{i \in I}$ is equivalent to the standard polarization over $\{U_i^{p(i)}\}_{i \in I}$ where $\U_i^+ = U_i$ and $U_i^- = X -U_i$. 

For the weights $\{a_i\}_{i\in I}$ and offset $b$ the polarization of the regions is dependent on $p$ such that $p(i)=+$ if $a_i \geq0$ and $p(i)=-$ if $a_i<0$. Therefore the polarization of the regions in a weighted union, $\sigma(\{x > \tilde{b}\})$, is over the $p$ polarization of all regions. To convert our polarization to be over the standard polarization we may take the image, $r_p(\sigma(\{x > \tilde{b}\})$.

\begin{defi} 
A \emph{selection} $\aS$ of regions of a collection of sets indexed by $I_\aS \subset P(I)$ to be the union of the regions:
$$ \aS = \bigcup_{J \in \sigma^{-1}(\{ x>\tilde{b}\})} U_J .$$
\end{defi}

\subsection{Characterization of a Perceptron}

We now will characterize a perceptron as a series of weighted unions on top of a polarized arrangement of hyper planes. Figure \ref{bigfigurecharperceptron} shows the various subnetworks in a neural network with 2 inputs, 4 planar nodes in the first layer, 2 selection nodes in the second layer and a selection node in the output layer. 
\begin{figure}\label{bigfigurecharperceptron}
\centering
\includegraphics[width=4.5in]{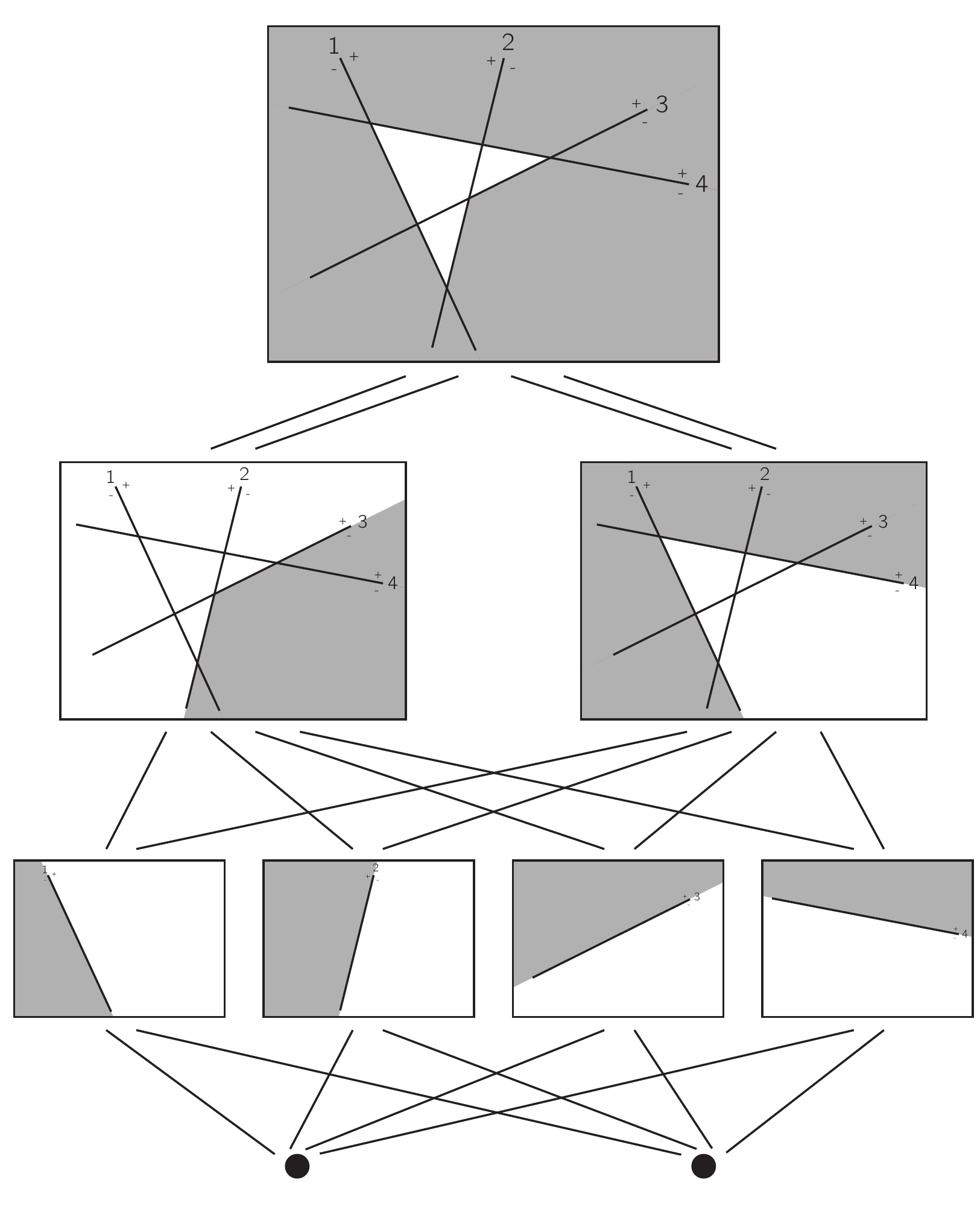}
\caption{A neural network with 4 plane nodes, 2 selection nodes and an output, which is also a selection node. For each node we have highlighted the area where the associated neural subnetwork is positive in grey.}
\end{figure}

\begin{lemma}\label{selectionpps}
A union or intersection of two selections of a plane partition set is another selection of that plane partition set.
\end{lemma}

\begin{theorem}\label{maintheorem}
A binary classification neural network is an indicator function on a union of some regions of an arrangement of hyperplanes $\aA$, where $\aA$ is defined by the first layer.
\end{theorem}
\begin{proof}
We will induct on the hidden height of a neural network. For the inductive step let $N$ be a neural network of height $k+1$. Each node in the final hidden layer of the neural network is a neural network of height $k$, let $U_i$ be the selection of regions for the $i$th node. Let the final node have weight $a_i$ from the $i$th node with offset $b_{k+1}$:
$$ N(x) = \ah \left ( \sum_i a_i \Chi_{U_i}(x) -b_{k+1}\right ) = \Chi_{\bigcup a_i U_i -b_{k+1}}(x). $$
As the weighted union is a finite number of intersections and unions on its the inputs and the inputs are all selections of the same regions by lemma \ref{selectionpps} the weighted union is. 
\end{proof}

We can see that if every selection in the penultimate layer pairs two regions i.e. they are both selected for or selected against together then the final selection node cannot separate the two. 

We can see that a neural network with $k$ outputs is going to be $k$ selections on the regions for the arrangement of hyperplanes. However, they are not independent from each other. The selections made by the last hidden layer affect the final possible selection.

\bibliographystyle{plain}
\bibliography{references}
\end{document}